\documentclass[letterpaper, 10 pt, conference]{ieeeconf} 
\IEEEoverridecommandlockouts                              % This command is only needed if 
\pdfoutput=1

\usepackage{bm}
\usepackage{amsmath,amssymb,amsthm}
\usepackage[table]{xcolor}
\usepackage[caption=false]{subfig}
\usepackage[]{mathtools}
\usepackage{bbm}
\usepackage{mdframed}
\usepackage[sort,compress]{cite} % Use [noadjust] instead?
\usepackage{tcolorbox}
\usepackage{algorithmicx}
\usepackage{algorithm}
\usepackage[noend]{algpseudocode}
\usepackage{setspace}

\usepackage{booktabs}
\usepackage{array}

\usepackage{overpic}

\algrenewcommand\algorithmicrequire{\textbf{Input:}}
\algrenewcommand\algorithmicensure{\textbf{Output:}}
\algnewcommand\algorithmicinput{\textbf{Input:}}
\algnewcommand\INPUT{\item[\algorithmicinput]}
\usepackage{moreverb,url}
\usepackage[colorlinks]{hyperref}
\usepackage{icomma} 
\usepackage{tikz} 
\usepackage{tkz-graph}
\usepackage{tkz-berge}
	\usetikzlibrary{shapes,snakes,arrows,shapes.geometric,positioning}
	\usetikzlibrary{intersections,patterns,shapes.misc}
	\usetikzlibrary{decorations.pathmorphing}
	\tikzstyle{block} = [rectangle, rounded corners, minimum width=3cm, minimum height=1cm,text centered, draw=black, fill=red!30]
	\tikzstyle{new} = [rectangle, rounded corners, minimum width=1cm, minimum
	height=1cm,text centered, draw=black, fill=blue!10!white, dashed]
	\tikzstyle{arrow} = [thick,->,>=stealth]
	\usetikzlibrary{calc, quotes}
\usepackage{pgfplots}
\usepgfplotslibrary{fillbetween}
\usepackage[perpage]{footmisc}
\usepackage{soul}
\usepackage{fontawesome}
\DeclareFontFamily{OT1}{pzc}{}
\DeclareFontShape{OT1}{pzc}{m}{it}{<-> s * [1.200] pzcmi7t}{}
\DeclareMathAlphabet{\mathpzc}{OT1}{pzc}{m}{it}

\AtBeginDocument{%
  \DeclareMathAlphabet\PazoBB{U}{fplmbb}{m}{n}%
}

\usepackage{dsfont}
\newtheorem{theorem}{Theorem}

\algrenewcommand\textproc{}%

\let\labelindent\relax

\let\oldthempfootnote\thempfootnote
\def\thempfootnote{\text{\oldthempfootnote}}
\usepackage{enumitem}

\setlist[itemize]{leftmargin=*}
\setlist[enumerate]{leftmargin=*}

\colorlet{lightgray}{green!4}

\usepackage[keeplastbox]{flushend}
\usepackage{color}
\usepackage{xparse}

\newcommand{\mbf}[1]{\mathbf{#1}}

\NewDocumentCommand\bbm{}{ \begin{bmatrix} }
\NewDocumentCommand\ebm{}{ \end{bmatrix} }

\include{matt_giamou_notation}

\newcommand{\jk}[1]{}

\newcommand{\mg}[1]{}% {{\color{green}MG: #1}}
 
\usepackage{soul}
\usepackage[outdir=./]{epstopdf}
\usepackage[para]{threeparttable}
\usepackage{multirow} 
\usepackage{standalone}

\title{\Large \bf Inverse Kinematics for Serial Kinematic Chains \\via Sum of Squares Optimization}

\author{Filip Mari\'c$^{a,b,\dagger }$, Matthew Giamou$^{a,c, \dagger}$, Soroush Khoubyarian$^a$, Ivan Petrovi\'c$^{b}$, and Jonathan Kelly$^a$
\thanks{ $^\dagger$ Denotes equal contribution.}
\thanks{ $^a$ Filip Mari\'c, Matthew Giamou, Soroush Khoubyarian, and Jonathan Kelly are with the University of Toronto, Institute for Aerospace Studies, Space and Terrestrial Autonomous Robotic Systems Laboratory, Canada. \{\texttt{<first name>.<last name>@robotics.utias.utoronto.ca}\}}% <-this % stops a space
\thanks{ $^b$ Filip Mari\'c and Ivan Petrovi\'c are with the University of Zagreb, Faculty of Electrical Engineering and Computing, Laboratory for Autonomous Systems and Mobile Robotics, Croatia. \{\texttt{<first name>.<last name>@fer.hr}\}}
\thanks{ $^c$ Vector Institute Postgraduate Affiliate and RBC Fellow.}
}
\begin{document} 
% Fixes the dashed-out author problem
\bstctlcite{IEEEexample:BSTcontrol}
\maketitle
%\thispagestyle{empty}
%\pagestyle{empty}

% Matt Pin:    182807
% Soroush Pin: 255878
% Jonathan Pin: 106824

% Paper headers
%\markboth{IEEE Robotics and Automation Letters. Preprint Version. Accepted TODO}{Maric \MakeLowercase{\textit{et al.}}: Inverse Kinematics for Serial Kinematic Chains via Sum of Squares Optimization} % Use only for final RAL version

%%%%%%%%%%%%%%%%%%%%%%%%%%%%%%%%%%%%%%%%%%%%%%%%%%%%%%%%%%%%%%%%%%%%%%%%%%%%%%%%
\begin{abstract}
Inverse kinematics is a fundamental challenge for articulated robots: fast and accurate algorithms are needed for translating task-related workspace constraints and goals into feasible joint configurations.
In general, inverse kinematics for serial kinematic chains is a difficult nonlinear problem, for which closed form solutions cannot easily be obtained.
Therefore, computationally efficient numerical methods that can be adapted to a general class of manipulators are of great importance. % to motion planning and workspace generation tasks.
 In this paper, we use convex optimization techniques to solve the inverse kinematics problem with joint limit constraints for highly redundant serial kinematic chains with spherical joints in two and three dimensions.
This is accomplished through a novel formulation of inverse kinematics as a nearest point problem, and with a fast sum of squares solver that exploits the sparsity of kinematic constraints for serial manipulators.
Our method has the advantages of post-hoc certification of global optimality and a runtime that scales polynomially with the number of degrees of freedom. 
Additionally, we prove that our convex relaxation leads to a globally optimal solution when certain conditions are met, and demonstrate empirically that these conditions are common and represent many practical instances. 
Finally, we provide an open source implementation of our algorithm. %and experiments.
\end{abstract}

%%%%%%%%%%%%%%%%%%%%%%%%%%%%%%%%%%%%%%%%%%%%%%%%%%%%%%%%%%%%%%%%%%%%%%%%%%%%%%%%
% Keywords appear just beneath the abstract. Use only for final RAL version.  

% Do these have to come from the list on the submission? 
% Use these for RAL submission
%\begin{IEEEkeywords}
%    Kinematics, Optimization and Optimal Control, Manipulation Planning
%\end{IEEEkeywords}

\section{Introduction} \label{sec:intro}
%
% Use this for final RA-L submission
%\IEEEPARstart{M}{any} 
Many common robots (e.g., manipulator arms and snake-like robots) can be modelled as \textit{kinematic chains}: rigid bodies connected by revolute joints that constrain robot motion to a specific workspace.
The motion of these robots may also be constrained by joint limits or user and task-specified workspace constraints on the positions or orientations of links.
Planning and controlling motion therefore requires solving the \textit{inverse kinematics} (IK) problem: finding configurations of the kinematic chain that satisfy a set of kinematic constraints.
A wide variety of techniques have been developed with the goal of solving IK for specific types of kinematic chains, such as manipulators with up to six degrees of freedom (DoFs).
However, generic solvers primarily rely on nonlinear optimization techniques, which typically solve the problem locally around an initial `seed' configuration.
Due to their local nature, these solvers cannot guarantee a feasible solution will be found, and as such may lead to the false conclusion that a problem is infeasible.

Kinematic chains are often parametrized using joint angles as variables, generating an IK problem comprised of nonlinear trigonometric equations.
However, alternative parameterizations exist that result in the kinematic equations taking on forms suitable for solution with a wider variety of mathematical tools.
Porta et al.~\cite{porta2005inverse} show that IK can be formulated as a special case of the \textit{Distance Geometry Problem} (DGP)~\cite{dattorro2010convex}, which consists of finding points that satisfy a given set of assigned distances.
Semidefinite programming (SDP) and sum of squares (SOS) relaxations are convex optimization techniques that have been used to solve the DGP in the domains of sensor network localization (SNL)~\cite{so2007theory, nie2009sum} and protein folding~\cite{alipanahi2013determining}.
In this paper, we demonstrate that the DGP in~\cite{porta2005inverse} can be represented as a \textit{quadratically constrained quadratic program} (QCQP), which can be extended to include other constraints such as joint limits. 
The main contributions of our work are:
 
\begin{enumerate} 
\item a polynomial formulation of IK with joint limit constraints, which admits provably tight SDP relaxations for problem instances which meet a criterion we characterize; 
\item a fast solution method for our formulation that uses a sparse SOS solver; and 
\item an open source implementation and experimental analysis of our algorithm in MATLAB.\footnote[1]{See \url{https://github.com/utiasSTARS/sos-ik} for code and supplementary material.}
\end{enumerate}

\begin{figure}
  \includegraphics[width=\columnwidth]{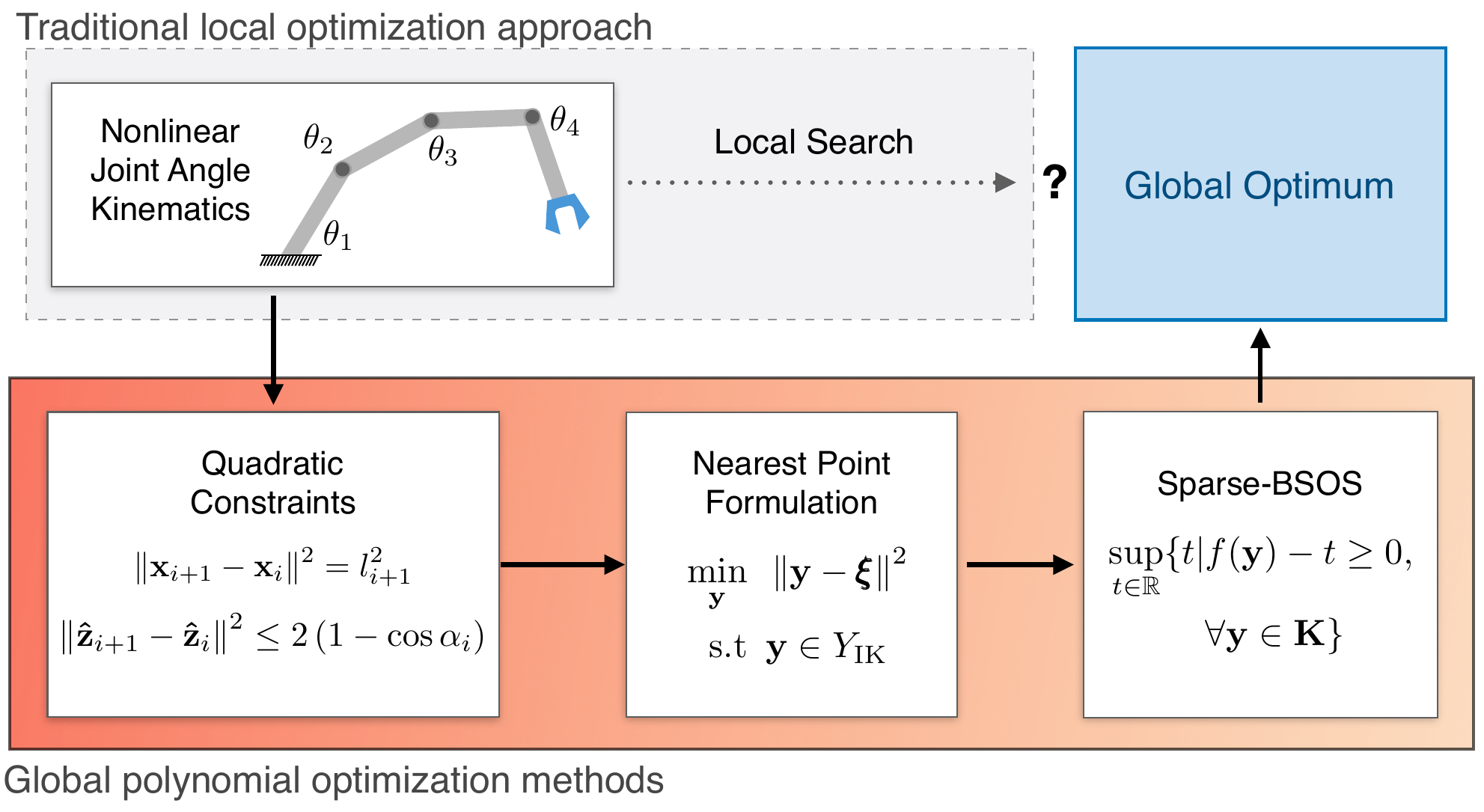}
  \caption{System diagram summarizing our approach. We parameterize manipulator kinematics with joint locations, leading to a QCQP problem formulation. This allows us to leverage  polynomial optimization techniques to obtain certifiably globally optimal solutions or certificates of infeasibility.} \label{fig:flow_chart}
\vspace{2mm}
\end{figure}

\section{Related Work} \label{sec:relatedWork}

In this section we review the two fields at whose intersection our paper lies: IK and global polynomial optimization.

\subsection{Inverse Kinematics}
Due to its widespread use, IK is a subject of intense research with an abundance of relevant literature that we can only briefly summarize in this paper; see \cite{aristidou2018inverse} for a recent, in-depth survey. 
Classical theoretical results~\cite{lee1988new} show that a general 6-DoF spatial kinematic chain has up to 16 configurations corresponding to any given feasible end-effector pose.
In such cases, closed form solutions can be found analytically using various parametrizations~\cite{manocha1994efficient,husty2007new,qiao2010inverse}, and software libraries such as \textit{IKFast}~\cite{diankov2010automated} can be used to rapidly generate feasible configurations.
One downside of these approaches is that they only account for other kinematic constraints (such as joint limits) as a post-processing step.
Moreover, for spatial (planar) kinematic chains with more than six (three) DoFs there exist an infinite number of solutions for any given end-effector pose, which means the state space for the redundant DoFs needs to be enumerated in a discrete and computationally expensive fashion.

When solutions cannot be obtained analytically, numerical methods are often used.
So-called closed-loop IK (CLIK) techniques use the Jacobian's (pseudo)inverse to apply differential kinematics in a closed-loop fashion, viewing IK as a feedback control problem~\cite{sciavicco1986coordinate}.
Moreover, CLIK methods using variants of damped least squares~\cite{buss2005selectively} and null-space optimization~\cite{nakamura1987task} provide numerically stable redundancy resolution for up to several hierarchical criteria.
Non-convex nonlinear optimization techniques such as sequential quadratic programming~\cite{schulman2014motion} iteratively generate convex subproblems, which can be efficiently solved.
These numerical methods do not provide any global optimality or feasibility guarantees, however, and therefore may require a large number of initializations to retrieve a feasible solution.
%
%Moreover, their performance scales terribly as the number of redundant DoFs increases.
%
The authors of~\cite{dai2017global} use a piecewise-convex relaxation of the SO(3) group to formulate the constrained IK problem as a mixed integer linear program (MILP).
They show that, unlike local optimization, their method requires no initialization and can provide a global certificate of infeasibility when a solution cannot be found.
Without approximating the SO(2/3) group,  our approach utilizes the theoretical result from~\cite{cifuentes2017local} to derive a tight SOS relaxation of IK for planar and spatial spherical kinematic chains, while retaining the ability to certify the global infeasibility of the problem.
%
%Unlike the aforementioned method, our approach utilizes the theoretical result from~\cite{cifuentes2017local} to derive a tight SOS relaxation of IK for planar and spatial spherical kinematic chains, which can also provide infeasibility certificates.
%
Moreover, our SOS relaxation leverages the innate sparsity pattern of the kinematic model to efficiently scale to a high number of DoF.
Finally, in~\cite{blanchini_inverse_2015} and~\cite{blanchini_convex_2017}, the authors analyze a convex formulation of inverse kinematics that is similar to our planar case but uses a linear cost function.

\subsection{Global Polynomial Optimization}

Optimization problems with polynomial cost and constraints are amenable to a host of convex relaxations that provide globally optimal solutions or bounds on the minimum cost. These methods include SDP relaxations for QCQPs~\cite{ma2010semidefinite}, as well as the broader class of global polynomial optimization techniques known as SOS programming~\cite{parrilo2003semidefinite, lasserre2001global}. While SOS programming technically involves solving an SDP, we will reserve the term ``SDP relaxation" for relaxations of the form described in Section~\ref{sec:proof}. Convex relaxations of polynomial optimization problems have found success in applications spanning signal processing, finance, control theory, and state estimation~\cite{boyd2004convex, lasserre2010moments, olsson2008solving}.
 
The literature on convex relaxations for SNL is closely related to the techniques developed in this paper. Most work focuses on the performance of SDP formulations for localization problems where noisy measurements of inter-sensor distances are provided~\cite{biswas2006semidefinite}. In this paper, we seek to solve inverse kinematics, where `measurements' are the distances between points of the kinematic chain. Thus, our approach is more closely related to the analyses of the noise free SNL problem found in~\cite{so2007theory} and~\cite{nie2009sum}, which demonstrate tightness of SDP and SOS relaxations, respectively. However, our work employs joint angle limits unique to robotics problems such as manipulation and uses a novel nearest-point formulation that has theoretical guarantees for relaxation tightness~\cite{cifuentes2017local}. 

Convex relaxations of polynomial optimization problems have been utilized in a variety of planning and control algorithms. In~\cite{deits2015computing}, convex obstacle-free regions of space in cluttered environments are efficiently generated via two alternating optimizations. This approach is used in~\cite{deits2015efficient} and~\cite{landry2016aggressive} to create SOS constraints enforcing collision-free trajectories in a mixed-integer planning approach to quadrotor flight. In~\cite{paden2017verification}, admissible heuristics for kinodynamic path planning problems are constructed with a SOS approximation method. In~\cite{jasour2015semidefinite}, chance-constrained formulations of optimization problems are introduced and solved via SOS programming. These methods are applied to problems in robotic motion planning and control to design trajectories  with bounded collision probabilities~\cite{jasour2016convex, jasour2018moment}.  Our algorithm, while not solving the entirety of a path planning or control problem, is complementary to these works and holds promise as a means of extending various planning methods to complex, high-dimensional kinematic models.

\section{Kinematic Model}\label{sec:problem_formulation} 
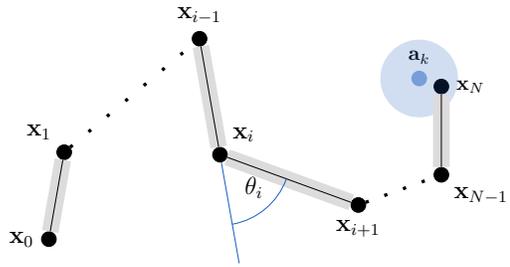
\begin{figure}
  \centering
  \resizebox{.8\columnwidth}{!}{\tikzset{
    position/.style args={#1:#2 from #3}{
        at=(#3),  shift=(#1:#2)
    }
  }
\begin{tikzpicture}
  [
    scale=3,
    >=stealth,
    point/.style = {draw, circle,  fill = black, inner sep = 2.5pt},
    bpoint/.style = {draw, circle,  fill = black, inner sep = 2.5pt},
    dot/.style   = {draw, circle,  fill = black, inner sep = .2pt}
  ]

  \definecolor{lightgr}{RGB}{220,220,220}
  \definecolor{lightbl}{RGB}{25,94,194}
  % the chain
  \def\rad{1}
  \def\link{\draw [double distance=1.5mm, very thick] (0,0)--}
  \def\centerarc[#1](#2)(#3:#4:#5)% Syntax: [draw options] (center) (initial angle:final angle:radius)
    { \draw[#1] ($(#2)+({#5*cos(#3)},{#5*sin(#3)})$) arc (#3:#4:#5); }

  \node (origin) at (0,0) [point, label = {left:\large{$\mathbf{x}_0$}}]{};
  \node [point, position=80:1.5 from origin, label = {above left:\large{$\mathbf{x}_1$}}] (n1) {};
  \node [point, position=40:3 from n1, label = {above:\large{$\mathbf{x}_{i-1}$}}] (n2) {}; 
  \node [point, position=-80:2 from n2, label = {above right:\large{$\mathbf{x}_{i}$}}] (n3) {};
  \node [point, position=-20:2.5 from n3, label = {below:\large{$\mathbf{x}_{i+1}$}}] (n4) {};
  \node [point, position=20:1.5 from n4, label = {below right:\large{$\mathbf{x}_{N-1}$}}] (n5) {};
  \node [point, position=90:1.5 from n5, label = {right:$\mathbf{x}_{N}$}] (n6) {};
  \node [point, fill=lightbl, draw=lightbl, opacity = 0.5, position=160:0.4 from n6, label = {above:$\mathbf{a}_k$}] (goal) {};
  \node [position=-80:2 from n3] (nj) {};
  \node [position=-80:1 from n3, label = {above right:\large{$\theta_i$}}] (nj2) {};

  \fill[fill=lightbl, opacity=0.2] (goal) circle [radius=0.22]; 

  % joint limit line 
  %\node [position=-80:2.5 from n3] (na) {};

  \draw[line width=8pt, draw=lightgr] (origin) -- (n1);
  \draw[] (origin) -- (n1);
  \draw[dash pattern=on \pgflinewidth off 10pt,ultra thick] (n1) -- (n2);
  \draw[line width=8pt, draw=lightgr] (n2) -- (n3);
  \draw[] (n2) -- (n3);
  \draw[line width=8pt, draw=lightgr] (n3) -- (n4);
  \draw[] (n3) -- (n4);
  \draw[dash pattern=on \pgflinewidth off 8pt,ultra thick] (n4) -- (n5);
  \draw[line width=8pt, draw=lightgr] (n5) -- (n6);
  \draw[] (n5) -- (n6);
  \draw[draw=lightbl, line width=0.1pt] (n3) -- (nj);
  \centerarc[draw=lightbl, line width=0.1pt](n3)(-80:-20:0.4)

\end{tikzpicture}}
  \caption{A kinematic chain comprised of $N-1$ spherical joints, and a virtual joint at the endpoint. Note that spherical joints in two dimensions correspond to revolute joints with a single rotation axis. The dotted lines represent links which are not shown.
    All constraints imposed on the joint $\mathbf{x}_i$ are a function of its position and the positions of nearby joints.
    The vertex $\mathbf{x}_N$ representing the endpoint is constrained to a ball around the point $\mathbf{a}_k$.}
  \label{fig:model_1}
  \vspace{2mm}
\end{figure}
In this section, we build on~\cite{porta2005inverse} to devise a model of the kinematic chain for which the IK problem can be formulated as a QCQP.
The resulting QCQP admits an SDP relaxation~\cite{ma2010semidefinite}, for which we provide sufficient conditions for tightness in Section~\ref{sec:proof}.
We also comment on the sparsity pattern of the constraints, which allows us to derive an efficient sparse SOS relaxation  (see Section~\ref{sec:rip} for details).
For clarity, we restrict our analysis to planar and spatial serial kinematic chains with $N$ spherical joints connected by $N$ straight rigid links.
However, the model presented here admits complex kinematic chains that include industrial and parallel manipulators~\cite{porta2005inverse}.

We begin by representing key points in the kinematic chain as vertices of a graph embedded in $\mathbb{R}^{d \geq 2}$.
In Figure~\ref{fig:model_1}, we can see that each joint of the chain is represented by a vertex $\mathbf{x}_i\,\in \mathbb{R}^{d\geq 2}\,, i = 1, 2, \dots, N$, with the vertices $\mathbf{x}_0$ and $\mbf{x}_N$ corresponding to the base and the endpoint respectively. 
Note that the full set of joint angles can be geometrically recovered from this representation.
\subsection{Distance Constraints}\label{sec:distance_constraints}
We can restrict the distance between two vertices
$\mathbf{x}_{i}$ and $\mathbf{x}_{j}$ to some range $D_{ij} = \left[D_{ij_{min}}, D_{ij_{max}}\right]$ by introducing the non-convex quadratic constraint,
\begin{equation}\label{eq:vertex_dist_constraint}
  D_{ij_{max}}^2 \geq \|\mathbf{x}_i-\mathbf{x}_j\|^2  \geq  D_{ij_{min}}^2\, ,
\end{equation}
which can be used to model the kinematic chain structure by constraining the distances between joints.
We can model rigid links by reducing the range $D_{i,i+1}$ to a single value $l_{i+1}$, corresponding to the length of the link between two consecutive joints at $\mathbf{x}_{i+1}$ and $\mathbf{x}_{i}$.
This results in an equality constraint, which restricts the distance between two joints to match the length of the link connecting them:
\begin{align}\label{eq:link_constraint}
  \begin{split}
  l_{i+1}^2 \geq& \|\mathbf{x}_{i+1}-\mathbf{x}_{i}\|^2  \geq  l_{i+1}^2\, \\
  \Leftrightarrow \quad &\|\mathbf{x}_{i+1}-\mathbf{x}_{i}\|^2 = l_{i+1}^2.
  \end{split}
\end{align} 
Without additional constraints, the vertex $\mathbf{x}_{i+1}$ is restricted to an $S^{d-1}$ sphere of radius $l_{i+1}$, centered at $\mathbf{x}_{i}$.
This corresponds to an unconstrained spherical joint in $\mathbb{R}^{d}$.

\subsection{Position Constraints}\label{sec:position_constraints}
The distance of any vertex from a fixed point in $\mathbb{R}^{d}$ (or \textit{anchor}) $\mathbf{a}_k$ can be restricted to the range $D_{ik} = \left[D_{ik_{min}}, D_{ik_{max}}\right]$ using
\begin{equation}\label{eq:vertex_anchor_constraint}
 D_{ik_{max}}^2 \geq  \|\mathbf{x}_i-\mathbf{a}_k\|^2 \geq  D_{ik_{min}}^2.
\end{equation}
Similarly to Eq.~\eqref{eq:link_constraint}, collapsing the range $D_{ik}$ in Eq.~\eqref{eq:vertex_anchor_constraint} to zero restricts the position of a vertex $\mathbf{x}_i$ to the point $\mathbf{a}_k$:
\begin{align}\label{eq:position_constraint} 
  \begin{split}
  0 \geq& \|\mathbf{x}_{i}-\mathbf{a}_{k}\|^2  \geq 0\, \\
  \Leftrightarrow \quad &\|\mathbf{x}_{i}-\mathbf{a}_{k}\|^2 = 0. 
  \end{split} 
\end{align}
This allows us to define the base position by constraining $\mathbf{x}_0$, as well as the exact pose (position and orientation) of the final link by constraining $\mathbf{x}_{N-1}$ and $\mathbf{x}_{N}$.

\subsection{Angle Constraints}\label{sec:angle_constraints}
\begin{figure}
  \centering 
  \resizebox{0.45\columnwidth}{!}{\tikzset{
    position/.style args={#1:#2 from #3}{
        at=(#3),  shift=(#1:#2)
    }
  }
\begin{tikzpicture}
  [
    scale=1.25,
    >=stealth,
    ] 

  \definecolor{lightgr}{RGB}{220,220,220}
  \definecolor{lightbl}{RGB}{25,94,194}
  % the chain
  \def\rad{1}
  \def\link{\draw [double distance=1.5mm, very thick] (0,0)--}
  \def\centerarc[#1](#2)(#3:#4:#5)% Syntax: [draw options] (center) (initial angle:final angle:radius)
    { \draw[#1] ($(#2)+({#5*cos(#3)},{#5*sin(#3)})$) arc (#3:#4:#5); }
  
  \node (origin) at (0,0) [inner sep=0,outer sep=0, label ={ [xshift=9pt, yshift=-3pt]\tiny{$\color{lightbl} \theta_i$}}]{};

  \node (origin2) at (0,0) [inner sep=0,outer sep=0, label ={ [xshift=3pt, yshift=4pt]\tiny{$\color{red} \alpha_i$}}]{};
  
    \node (origin3) at (0,0) [inner sep=0,outer sep=0, label ={ [xshift=6pt, yshift=-11pt]\tiny{$\color{red} \alpha_i$}}]{};

  \node [position=0:0.5 from origin, inner sep=0, outer sep=0, label={[xshift=12pt, yshift=-12pt]\tiny{$\mbf{\hat{z}}_{i+1}$}}] (n1) {};

  %\node [position=25:1 from origin, inner sep=0, outer sep=0, label={[xshift=11pt, yshift=-5pt]\tiny{$ \color{lightbl}  \sqrt{ 2\left(1 - \cos\theta_i\right)} $}}] (nn) {};
  
  \node [position=50:0.5 from origin, inner sep=0,outer sep=0, label={[xshift=5pt, yshift=7pt]\tiny{$\color{lightbl} \mathbf{\hat{z}}_i$}}] (n2) {}; 

  \draw[line width=0.1pt, opacity = 0.5] (origin) circle [radius=1];
  \centerarc[lightbl](origin)(0:50:0.4)
  \centerarc[red](origin)(0:100:0.45)
  \centerarc[dashed,red, opacity=0.5](origin)(0:-100:0.45)

  % joint limit line 
  \node [position=-100:1 from origin, inner sep=0,outer sep=0] (na1) {};
  \node [position=100:1 from origin, inner sep=0,outer sep=0] (na2) {};

  \draw[->] (0,0) -- (0:1);
  \draw[lightbl,->] (0,0) -- (50:1);
  \draw[thin, lightbl] (0:1) -- (50:1);
  \draw[dashed, red, ultra thin, opacity=0.5] (0,0) -- (-100:1);
  \draw[red, ->] (0,0) -- (100:1);
  \draw[thin,red] (100:1) -- (0:1);

\end{tikzpicture}}
  \caption{Visualization of the convex angle constraint in Eq.~\eqref{eq:angle_constraint}. As both the vectors $\hat{\mbf{z}}_i$ and $\hat{\mbf{z}}_{i+1}$ are of unit length, the length $\left\lVert \hat{\mbf{z}}_i - \hat{\mbf{z}}_{i+1}\right\rVert$ depends only on the angle between them.}
  \label{fig:joint_limits}
  \vspace{2mm}
\end{figure}
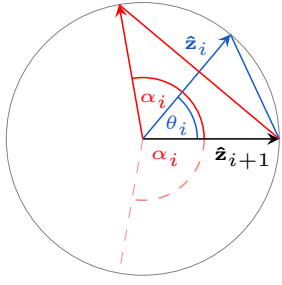
The angle $\theta_i$ of any joint $\mathbf{x}_{i}$ with respect to its parent joint $\mathbf{x}_{i-1}$ is commonly limited by mechanism design. In Figure~\ref{fig:joint_limits}, the unit vectors $\hat{\mbf{z}}_i = \frac{1}{l_i} (\mbf{x}_i - \mbf{x}_{i-1})$ and $\hat{\mbf{z}}_{i+1}$ are related to joint angle $\theta_i$ and limit $\alpha_i$. Applying the cosine law leads to the equivalence
\begin{align}\label{eq:angle_constraint} 
\begin{split} 
	&|\theta_i| \leq \alpha_i \\   
	\Leftrightarrow \quad &\left\lVert \hat{\mbf{z}}_{i+1} - \hat{\mbf{z}}_{i}\right\rVert^2 \leq
  2\left(1 - \cos\alpha_i\right), 
\end{split}
\end{align}
which can be used to enforce joint limit constraints symmetric with respect to the previous link, as shown in Figures~\ref{fig:model_1} and~\ref{fig:joint_limits}.
In~\cite{blackmore2006optimal}, it is noted that quadratic constraints can also be used for non-symmetric angle ranges smaller than $180^{\circ}$

Note that the constraints described in this section form a sparsity pattern: the position of any joint $\mathbf{x_{i}}$ only appears in constraints  with nearby joints $\mathbf{x_{i-k}}$ and $\mathbf{x_{i+k}}$ for $k \leq 2$.
In Section~\ref{sec:rip} we explain how this sparsity can be exploited by an SOS solver to efficiently find IK solutions.
This kinematic model can also be extended to include other quadratic constraints such as collision avoidance, which we plan to explore in future work. 

\section{Inverse Kinematics Formulation} \label{sec:inverseKinematics}

\newcommand{\fJacobian}{\nabla f(\bar{\pmb{\xi}}))}
\newcommand{\acq}{\mathrm{rank}(\fJacobian) = n - \mathrm{dim}_{\bar{\pmb{\xi}}} Y} 

In this section, we cast IK as an optimization problem seeking the feasible configuration whose joints are closest to some target positions in the workspace. This \emph{nearest point} formulation of IK allows us to prove Theorem \ref{thm:stability}, which sheds light on the globally optimal performance of our convex relaxations. 

\subsection{Algebraic Variety of Feasible Configurations} \label{sec:algebraicVariety}
\newcommand{\kinematicsY}{Y_{\text{IK}}}
The \emph{variety} Y of a set of polynomial equations $f_i(\mbf{y}) = 0$ is the set of real-valued solutions satisfying those equations:
\begin{equation}
	Y \coloneqq \{\mbf{y} \in \mathbb{R}^n : f_1(\mbf{y}) = \cdots = f_m(\mbf{y}) = 0\}.
\end{equation}
In order to define the set of all kinematically feasible $N$-link chains that connect the origin in $\mathbb{R}^d$ to a desired end position $\mathbf{x}_N$ as a variety, we need to express the inequalities representing angle constraints from Section \ref{sec:angle_constraints} as equalities. To this end, we introduce $N$ auxiliary variables $s_i$ \cite{park2017general} and note that any inequality constraints satisfy the equivalence 
\begin{equation} 
f_i(\mathbf{x}) \leq 0 \iff f_i(\mathbf{x}) + s_i^2 = 0.
\end{equation} 
We will use $\mbf{x} \in \mathbb{R}^{d(N-1)}$ to denote the concatenation of `interior' joints $\mbf{x}_i,\ i=1,\ldots, N-1$, and $\mbf{s} \in \mathbb{R}^N$ to denote the column vector of auxiliary $s_i$ variables.
We can now define our kinematically feasible set as the algebraic variety
\begin{equation}
\kinematicsY \coloneqq \{\mbf{y} \in \mathbb{R}^n : g_i(\mbf{x}) = h_i(\mbf{y}) = 0, i=1, \ldots, N \}, \\
\end{equation}
 where $n = d(N-1)+N$ and $\mbf{y} = [\mbf{x}^T \ \mbf{s}^T]^T$, and
 \begin{align}
 \begin{split}
 	g_i(\mbf{x}) &= \|\mathbf{x}_{i}-\mathbf{x}_{i-1}\|^2 -l_{i}^2, \\
 	h_i(\mbf{y}) &= \left\lVert \hat{\mbf{z}}_{i+1} - \hat{\mbf{z}}_{i}\right\rVert^2 + s_i^2 -
  2\left(1 - \cos\alpha\right).
 \end{split}
 \end{align}
 To summarize, the variety $\kinematicsY$ is the feasible set for a particular instance of IK parameterized by the number of DoFs $N$, the link lengths $l_i$, the angle limits $\alpha_i$, and the target pose of the final link $\mbf{x}_N$. This formulation assumes, without loss of generality, that $\mbf{x}_0 = \mbf{0}$.
 \subsection{Nearest Point Problem}
 For redundant manipulators, $\kinematicsY$ contains infinitely many solutions for nearly all target positions $\mbf{x}_N$.  The set described by variety $\kinematicsY$ is high-dimensional and nonconvex. In order to find solutions, we will cast IK as the problem of finding the \emph{nearest point} $\mbf{y} \in \kinematicsY$ to some reference point $\pmb{\xi} \in \mathbb{R}^{n}$. Since the squared Euclidean distance is used for the cost, and $\kinematicsY$ is a quadratic variety, this allows us to cast IK as a quadratically constrained quadratic program (QCQP):
\begin{align} \label{prob:nearestPoint}
\begin{split}
	\underset{\mbf{y}}{\min} \enspace & \left\lVert\mbf{y} - \pmb{\xi}\right\rVert^2, \\
	\text{s.t} \enspace & \mbf{y} \in \kinematicsY, 
\end{split}
\end{align}
where $\pmb{\xi} = [\mbf{x}_0^T \ \mbf{s}_0^T]^T \in \mathbb{R}^{d(N-1) + N}$. Note that the cost also includes the squared distance between the auxiliary variables $s_i$ and their reference points $s_{0,i}$ in $\pmb{\xi}$.

\subsection{Proof of Strong Duality} \label{sec:proof}
In this section, we prove that for many instances of the problem in Eq.~\ref{prob:nearestPoint} (Problem \ref{prob:nearestPoint}), the convex SDP relaxation is tight, and therefore we can find a global optimum of Problem \ref{prob:nearestPoint} in polynomial time with interior point solvers. We refer the reader to \cite{boyd2004convex} and \cite{cifuentes2017local} for detailed discussions of tightness, strong duality, and SDP relaxations. 
We begin by stating Theorem \ref{thm:nearestPoint} from \cite{cifuentes2017local}, which is a general result for nearest-point problems in polynomial optimization. 
\begin{theorem}[Nearest Point to a Quadratic Variety \cite{cifuentes2017local}] \label{thm:nearestPoint}
	Consider the problem 
	\begin{equation}
		\underset{\mbf{y} \in Y}{\min} \left\lVert\mbf{y} - \pmb{\xi}\right\rVert^2,
	\end{equation}
	where $Y \coloneqq \{\mbf{y} \in \mathbb{R}^n : f_1(\mbf{y}) = \cdots = f_m(\mbf{y}) = 0\}$,  $f_i$ quadratic.
	Let $\bar{\pmb{\xi}} \in Y$ be such that 
	\begin{equation}
		\acq.
	\end{equation}
	Then there is zero-duality-gap for any $\pmb{\xi} \in \mathbb{R}^n$ that is sufficiently close to $\bar{\pmb{\xi}}$.  
\end{theorem}

Note that Theorem \ref{thm:nearestPoint} references the duality gap, which is related to the Lagrangian dual relaxation. We use the fact that the zero-duality-gap property, also called \emph{strong duality}, implies tightness of the SDP relaxation \cite{cifuentes2017local}.  Since having access to an admissible $\bar{\pmb{\xi}} \in \kinematicsY$ amounts to having solved the IK problem already, we would like to be able to use an SDP relaxation of Problem \ref{prob:nearestPoint} to obtain a valid solution by using a reference point $\pmb{\xi} \notin \kinematicsY$. In general, most nonconvex problems do not exhibit strong duality. Problem \ref{prob:nearestPoint} contains non-convex link length constraints, making the existence of tight SDP relaxations a non-trivial and useful property.

\begin{theorem}[Strong Duality] \label{thm:stability}
	If $\bar{\pmb{\xi}} \in \kinematicsY$ does not represent a fully extended configuration (i.e., the joint positions are not all collinear), and does not have any joints at their angular limits for specified base and goal positions, then Problem \ref{prob:nearestPoint} exhibits strong duality for all $\pmb{\xi}$ sufficiently close to $\bar{\pmb{\xi}}$.   
\end{theorem}

The proof of Theorem \ref{thm:stability} can be found in our supplementary material. In Section \ref{sec:experiments}, we demonstrate that the tight-relaxation region for instances of Problem \ref{prob:nearestPoint} is substantial and randomly sampling $\pmb{\xi}$ is a practical strategy. \jk{mention later that we can characterize the region exactly as future work?}

\section{SOS Programming}\label{sec:SOS}
Sum of squares (SOS) programming is an approach for solving polynomial optimization problems with convex optimization. The standard SOS relaxation hierarchy \cite{parrilo2003semidefinite, lasserre2001global} is equivalent to the Lagrangian dual relaxation with particular redundant constraints added \cite{cifuentes2017local}. These redundant constraints can only make the relaxation tighter. Combined with the fact that the Lagrangian dual relaxation is a lower bound of the SDP relaxation of Problem \ref{prob:nearestPoint}, this tells us that the standard SOS hierarchy shares the stability property proved in Theorem \ref{thm:stability}. This work uses the Sparse-BSOS method of \cite{weisser2018sparse}, a recent sparse extension of \cite{lasserre2017bounded}, which introduced a SOS hierarchy that is less computationally costly than the standard SOS hierarchy in many cases, while remaining just as tight for QCQPs.

\subsection{Sparse-BSOS}\label{subsec:SBSOS}
For a complete treatment of the Sparse-BSOS hierarchy, please refer to \cite{weisser2018sparse}. Briefly, we are interested in solving Problem \ref{prob:nearestPoint} in the equivalent form
\begin{equation}
t^\star = \underset{t \in \mathbb{R}}{\sup} \{t | f(\mbf{y}) - t \geq 0, \ \forall \mbf{y} \in \mbf{K} \},
\end{equation}
where $\mbf{K} = \{\mbf{y} \in \mathbb{R}^{n} | 0 \leq g_j(\mbf{y}) \leq 1, \ j = 1, \ldots, m \}$ is a semialgebraic set equivalent to $\kinematicsY$ in Section \ref{sec:algebraicVariety}, and $f(\mbf{y}) = \left\lVert\mbf{y} - \pmb{\xi}\right\rVert^2$. The key insight of SOS optimization is that this problem (and other polynomial optimization problems) can be solved as a semidefinite program (SDP) with \emph{Positivstellensatz} results from real algebra \cite{lasserre2010moments, parrilo2003semidefinite}. Many SOS relaxation hierarchies have been developed, but we use the sparse bounded-degree SOS (Sparse-BSOS) hierarchy of \cite{weisser2018sparse} because it leverages the natural sparsity of kinematic chains. The method enforces $f(\mbf{y}) - t \geq 0$ by introducing the function 
\begin{equation}
\begin{aligned}
h_d(\mbf{y}, \pmb{\lambda}) &= \sum_{\alpha, \beta \in \mathbb{N}^m}^{|\alpha|_1 + |\beta|_1 \leq d} \lambda_{\alpha\beta} h_{d, \alpha\beta}(\mbf{y}),\\
h_{d,\alpha \beta}(\mbf{y}) &\vcentcolon = \prod_{j=1}^m g_j(\mbf{y})^{\alpha_j}(1 - g_j(\mbf{y}))^{\beta_j}, \ \mbf{y} \in \mathbb{R}^{n},
\end{aligned}
\end{equation}
where $\pmb{\lambda}$ contains the coefficients $\lambda_{\alpha\beta} \geq 0$ indexed by $\alpha$ and $\beta$, and the parameter $d$ allows us to restrict the number of monomials used to construct $h_d$. Now we seek to optimize 
\begin{equation}
t^\star = \underset{t, \pmb{\lambda}}{\sup} \{t | f(\mbf{y}) - t - h_d(\mbf{y}, \pmb{\lambda}) \geq 0, \ \forall \mbf{y}, \pmb{\lambda} \geq 0 \},
\end{equation} 
where $h_d(\mbf{y}, \pmb{\lambda}) > 0$ when $\mbf{y} \in \mbf{K}$ (see \cite{lasserre2017bounded} for details). Next, the problem is converted to an SDP by restricting the search to $\Sigma[\mbf{y}]_k$, the set of SOS polynomials of degree at most $2k$, which constitute a subset of nonnegative polynomials:
\begin{equation}\label{eq:sparse-bsos}
q_d^k = \underset{t, \pmb{\lambda}}{\sup} \{t | f(\mbf{y}) - t - h_d(\mbf{y}, \pmb{\lambda}) \in \Sigma[\mbf{y}]_k, \ \forall \mbf{y}, \pmb{\lambda} \geq 0 \}.
\end{equation} 
Each $q_d^k$ describes a level of the BSOS hierarchy indexed by $d$ and $k$ \cite{lasserre2017bounded}. Since Problem \ref{prob:nearestPoint} is a QCQP, $k=1$ in our use of the solver. 
Finally, to produce the Sparse-BSOS hierarchy we partition Problem \ref{eq:sparse-bsos} into smaller blocks of variables and relevant constraints. These subsets of variables must satisfy a sparsity property called the running intersection property (RIP). 

\subsection{The Running Intersection Property} \label{sec:rip}
\mg{Add notation for indexed subsets used here.} In order for the Sparse-BSOS hierarchy to converge to the global optimum as $d \rightarrow \infty$, the variables and functions involved must satisfy a sparsity property called the running intersection property (RIP) \cite{weisser2018sparse}. In describing the RIP we will use $[k]$ to denote the set $\{1, 2, \ldots, k\}$ compactly. The RIP holds if there exists $p \in \mathbb{N}$ and subsets $I_l \subseteq [n]$ and $J_l \subseteq [m]$ for all $l \in [p]$ such that:
\begin{itemize}
\item $f = \sum_{l=1}^p f^l$, for some $f^1, \ldots, f^p$ such that \\ $f^l \in \mathbb{R}[\mathbf{x}; I_l], \ l\in [p]$;
\item $g_j \in \mathbb{R}[\mathbf{x}; I_l]$ for all $j \in J_l$ and $l \in \{1, \ldots, p\}$; 
\item $\bigcup_{l=1}^p I_l = [n]$; 
\item $\bigcup_{l=1}^p J_l = [m]$;
\item for all $l \in [p-1]$ there exists $s \leq l$ such that \\ $(I_{l+1} \cap \bigcup_{r=1}^l I_r) \subseteq I_s$. 
\end{itemize}
For the case of a 2D manipulator with only link length constraints, the partition consisting of overlapping pairwise joints satisfies the RIP. When joint limit constraints are introduced, overlapping triplets of joints are required. \mg{Short proof here? Probably not needed.} For redundant manipulators with many links, this partition amounts to an SDP with far fewer variables and constraints than a standard SDP or SOS relaxation would generate. The SDP produced by Sparse-BSOS has semidefinite constraints on variables of size $O(n^\star)$ for $k=1$, where $n^\star = \max_l n_l$ and $n_l$ is the number of variables in $I_l$ \cite{weisser2018sparse}. For a $d$-dimensional manipulator using our nearest-point formulation of IK with a partition that satisfying the RIP, $n^\star = 3d + 1$. Using the Sparse-BSOS hierarchy therefore requires less memory and runtime as compared with its dense equivalent, whose semidefinite constraint variables would be of size $O(dN)$ (where $N$ is the number of degrees of freedom as in prior sections). Our entire algorithm is summarized in Figure~\ref{fig:flow_chart}.

\section{Experiments} \label{sec:experiments}
\newcommand{\rankOne}[1]{\mathcal{R}_1(#1)}
In this section, we present IK solutions for simulated planar (2D) and spatial (3D) manipulators.
All experiments were conducted with a MATLAB implementation of our approach on a computer with a 2.2GHz Intel Core i7-8750H CPU.
Please refer to the supplementary material for the kinematic chain link lengths and angle limits used in our experiments.
The primary purpose of our experiments is to explore our global method and its theoretical guarantee (Theorem \ref{thm:stability}).
We recognize that there may be local solvers that are competitive in some instances, but the focus of our work is                                                                                                                                                                                                                                                                                                                                                                                                           on the global optimality properties of convex optimization methods.
In order to keep our focus concise in this paper, we consider a full and thorough comparison that includes experiments on physical robots as part of a future, more comprehensive work.

\subsection{Global Optimality}\label{sec:glopt}

\begin{figure*}[t!]
    \centering
%    \subfloat[$\rankOne{\pmb{\xi}_1}$]{\includegraphics[height=2in]{fig/fig1}}
    \subfloat[$\rankOne{\pmb{\xi}_1}$]{
    
    \begin{overpic}[height=1.8in]{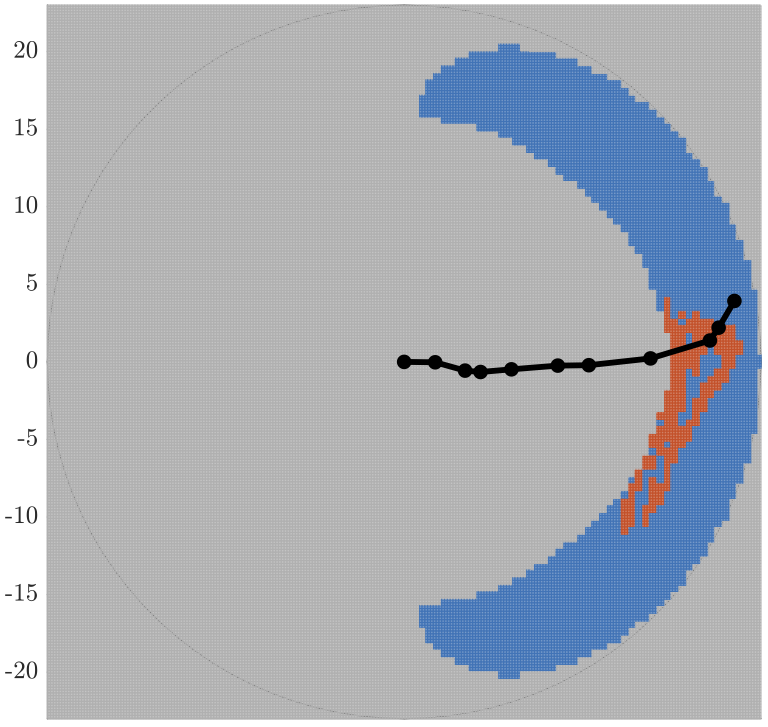}
     	\put(10,65){\includegraphics[height=0.5in]{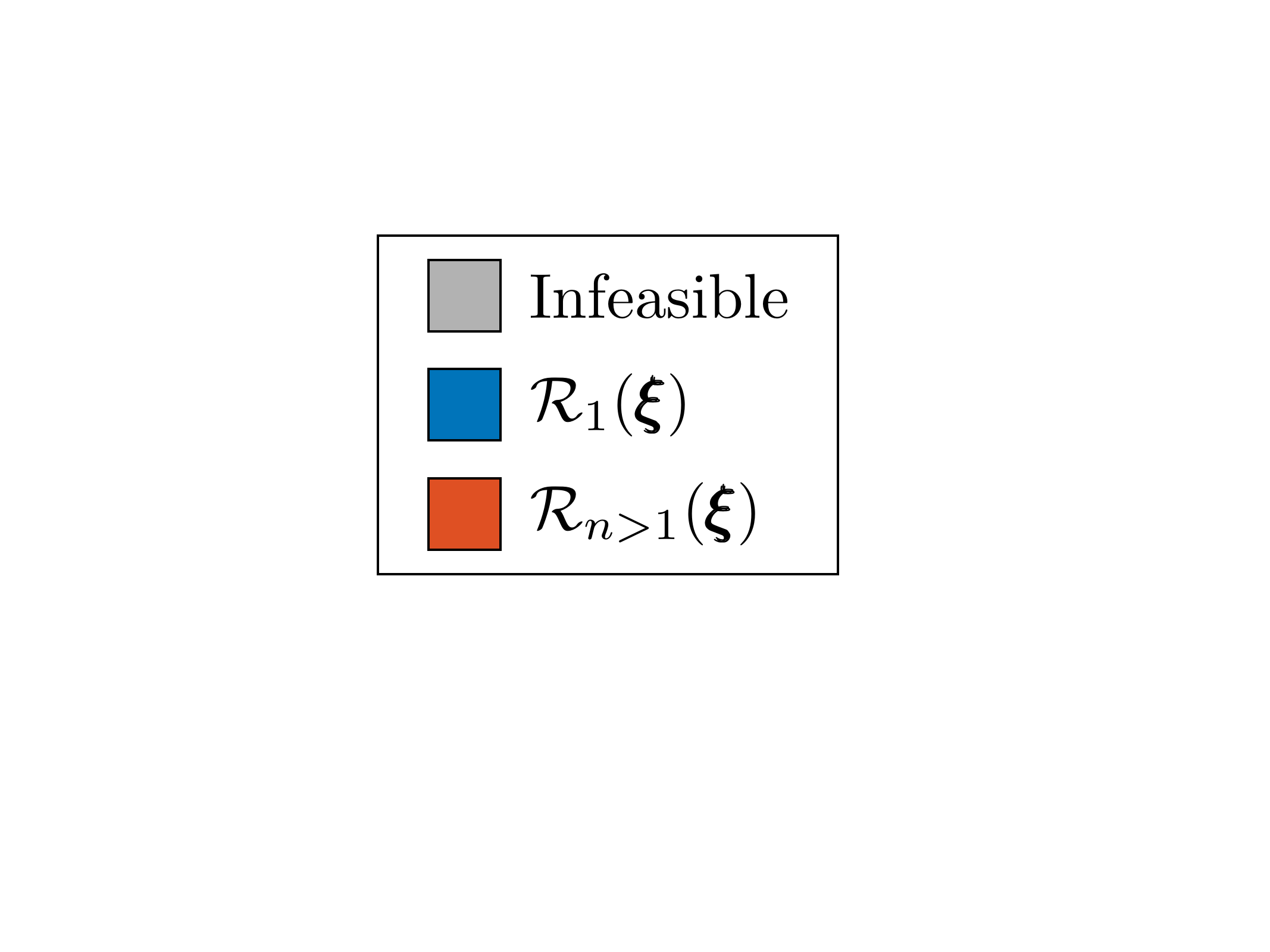}}  
  	\end{overpic}
    }
    \label{fig:rank1A}
    ~ 
    \subfloat[$\rankOne{\pmb{\xi}_2}$]{\includegraphics[height=1.8in]{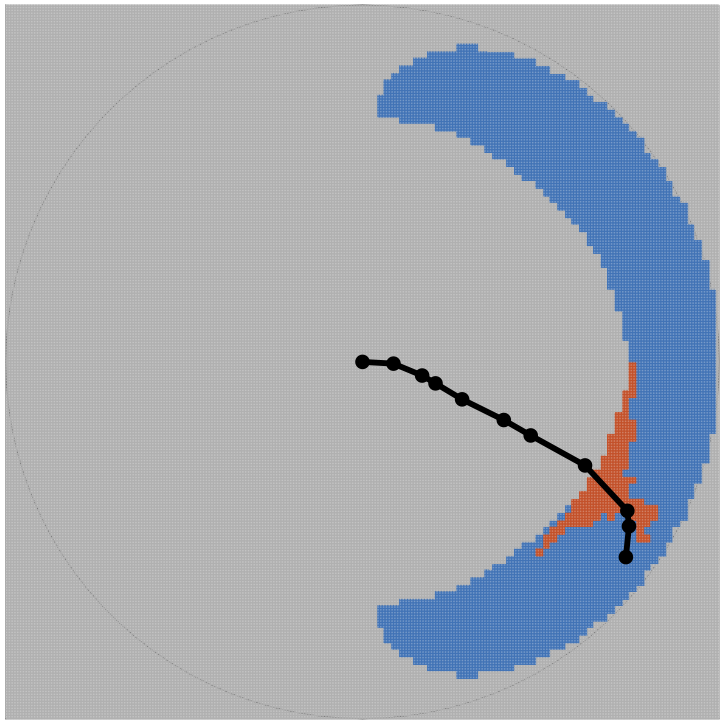}}
    \label{fig:rank1B}
    ~ 
    \subfloat[$\rankOne{\pmb{\xi}_3}$]{\includegraphics[height=1.8in]{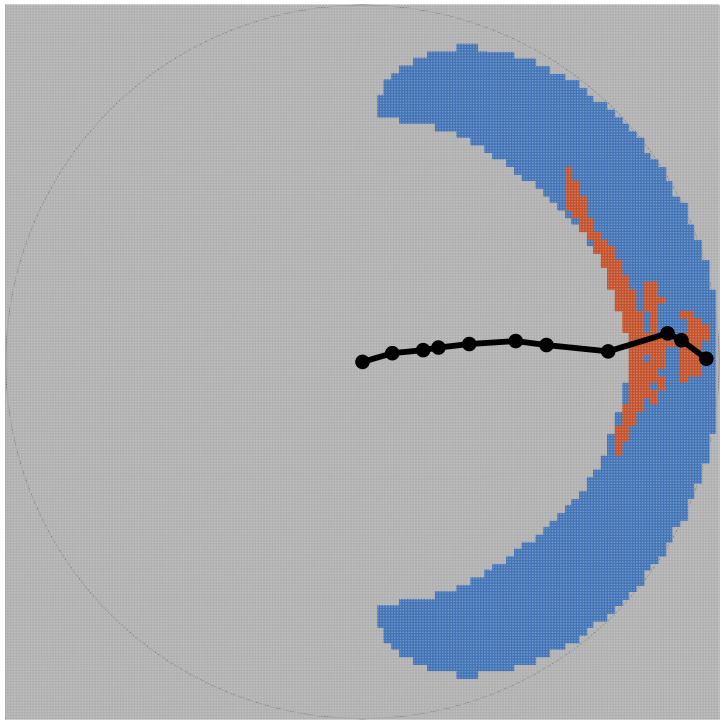}}
    \label{fig:rank1C}
    \\ \vspace{0.01in}	
    \hspace{0.001in}
    \subfloat[$\rankOne{\pmb{\xi}_4}$]{\includegraphics[height=1.88in]{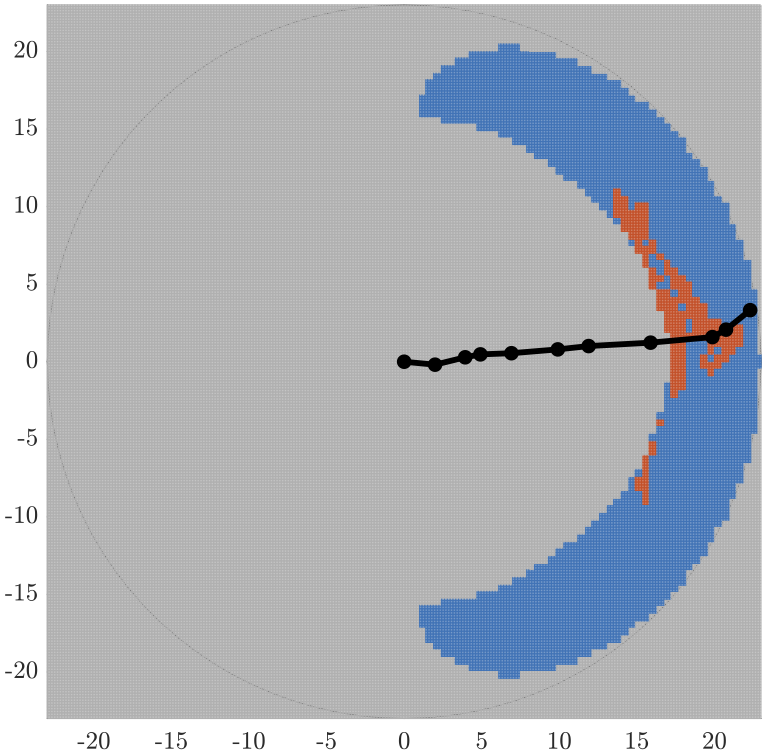}}
    \label{fig:rank1D}
    \hspace{0.01in} 
    ~ 
    \subfloat[$\rankOne{\pmb{\xi}_5}$]{\includegraphics[height=1.88in]{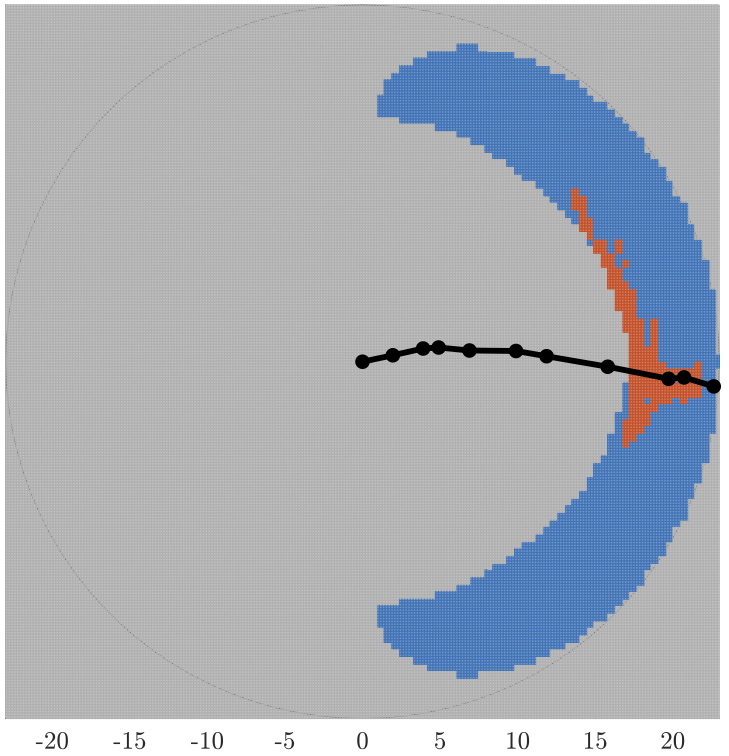}}
    \label{fig:rank1E}
    ~ 
    \subfloat[$\bigcup\limits_{i=1}^5 \rankOne{\pmb{\xi}_i}$]{\includegraphics[height=1.88in]{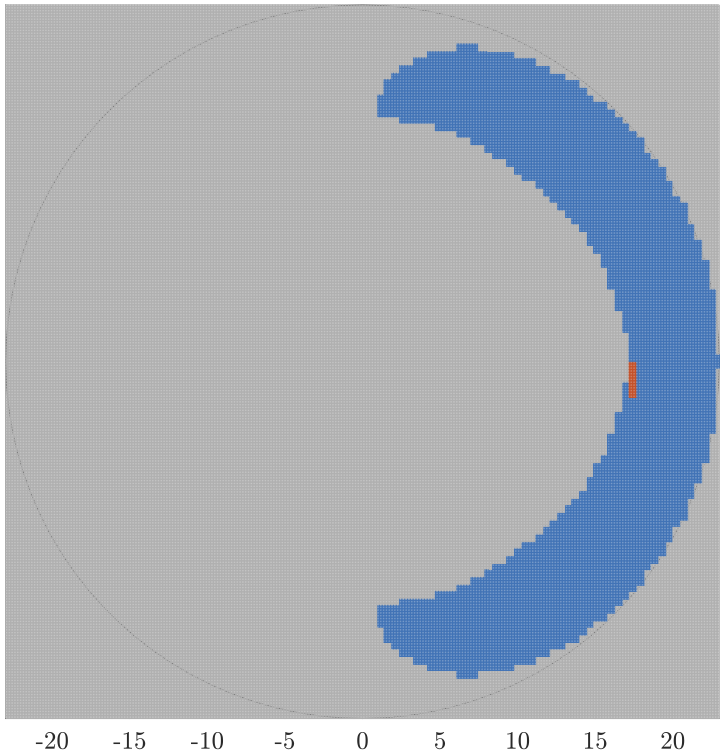}}
    \caption{Heatmaps demonstrating rank-1 regions $\rankOne{\pmb{\xi}}$ in target end-effector space  in blue for 5 different randomly sampled reference configurations $\pmb{\xi}_i$. Red regions indicate rank greater than 1, and grey indicates infeasible goal positions. In heatmaps (a)-(e), the joint positions for $\pmb{\xi}_i$ are plotted in black. Heatmap (f) displays the union of the rank-1 regions for all 5 reference configurations. See the supplementary material for manipulator parameters. \jk{can figures be a bit smaller? would solve space issue - also maybe change grey background to white}}
    \label{fig:unionRank}
    \label{fig:rankPlots}
\end{figure*}

In Figure~\ref{fig:rankPlots}, we display the rank-1 region of goal space in blue, which we denote $\rankOne{\pmb{\xi}}$ for various randomly sampled reference points $\pmb{\xi}_i$ (drawn in black). 
The red regions corresponds to goals that led to Sparse-BSOS solutions with rank greater than 1, whereas grey areas contain goals found to be infeasible.
Since Sparse-BSOS solves a relaxation of the primal problem in~\ref{prob:nearestPoint}, infeasibility in the relaxed problem's solution provides a certificate of infeasibility for the primal.
Additionally, the existence of a rank-1 solution is a certificate of optimality for the extracted solution.
Figure~\ref{fig:unionRank} displays the union of the rank-1 regions for all 5 randomly sampled $\pmb{\xi}_i$ and almost covers the entirety of the feasible goal space. 
  
\subsection{Inverse Kinematics for 2D and 3D Serial Chains}

\begin{table*}
 \begin{center}
   \caption{Performance of our method over 10,000 randomly generated feasible 2D (top) and 3D (bottom) pose goals.}
   \label{table:comparison_1}
   \resizebox{\textwidth}{!}{
   \begin{tabular}{cccccccccc}
     \toprule
     DoF & \multicolumn{2}{c}{5} & \multicolumn{2}{c}{7} & \multicolumn{2}{c}{10} & \multicolumn{2}{c}{12}  \\
     \cmidrule(r){2-3}  \cmidrule(r){4-5}   \cmidrule(r){6-7} \cmidrule(r){8-9}
     Method                 &SOS-IK                                & \texttt{fmincon}                   & SOS-IK                             & \texttt{fmincon}                    & SOS-IK                              & \texttt{fmincon} & SOS-IK  & \texttt{fmincon} \\
     \midrule
     \midrule
     Position Error [m]        &$3.51\times10^{-4}$                    & $6.78\times 10^{-5}$               &$3.17\times 10^{-4}$                  & $4.93\times 10^{-5}$                &$1.64\times 10^{-4}$                  & $2.69\times 10^{-4}$ &$2.00\times 10^{-4}$ &$2.67\times 10^{-4}$ \\ 
%     \cmidrule(r){2-10}
     Solved [$\%$]          &$97.46$                               & $97.15$                            &$99.52$                             & $93.33$                             &$99.47$                               & $87.65$&$98.11$ &$53.19$ \\
     Total Time [min]       &$9.08$                                & $1.67$                             &$13.73$                             & $1.93$                              &$20.05$                               & $3.52$ &$21.52$ &$3.6$ \\
     \midrule \midrule
     Position Error [m]        &$6.68\times10^{-7}$           & $2.18\times 10^{-5}$   &$1.33\times 10^{-6}$    & $3.79\times 10^{-5}$    &$2.64\times 10^{-5}$      & $6.93\times 10^{-4}$ & $1.41\times 10^{-5}$ & $3.1\times 10^{-3}$ \\
%     \cmidrule(r){2-10}
     Solved [$\%$]          &$99.29$                                & $99.67$               &$99.68$                             & $99.84$                &$98.51$                               & $95.79$ & $99.72$ & $94.55$ \\
     Total Time [min]       &$18.88$                              & $9.44$                &$25.36$                             & $23.23$                 &$38.22$                                   & $46.10$ & $33.92$ & $51.69$ \\
     \bottomrule
   \end{tabular}
   }
 \end{center} 
\end{table*}

Theorems~\ref{thm:nearestPoint} and \ref{thm:stability} prove that by formulating IK as Problem~\ref{prob:nearestPoint}, globally optimal solutions can be recovered in certain workspace regions using convex relaxations such as Sparse-BSOS.
The experiment in Section~\ref{sec:glopt} demonstrates that such workspace regions are quite large even for complex kinematic chains like redundant spherical manipulators in two and three dimensions, which don't admit analytical solutions in the presence of joint limits. 
By solving 10,000 feasible IK problems, we demonstrate that our method (dubbed SOS-IK) outperforms a local numerical \texttt{fmincon} implementation of a joint angle-based IK solver in MATLAB in terms of percentage of recovered solutions, while also providing post-hoc numerical certificates of problem (in)feasibility.
The results in Table~\ref{table:comparison_1} show the final end-effector position error, percentage of feasible IK solutions found, and total computation time over all problems for manipulators with an increasing number of DoF.

The upper half of Table~\ref{table:comparison_1} shows results for planar manipulators of increasing DoF, where SOS-IK outperforms the local optimization in the percentage of solved problems for every problem instance.
While the solve times for \texttt{fmincon} are significantly lower in the planar case, we note that our method has well-understood polynomial scaling properties which present themselves favourably for higher DoF and dimensionality.
This can be seen in the in the case of a 12 DoF planar manipulator, where SOS-IK finds almost twice as many feasible solutions than \texttt{fmincon}, which falls into local minima.
We show how this trend continues for spherical (3D) manipulators in the bottom of Table~\ref{table:comparison_1}. 
As DoF increases, solve times become comparable and SOS-IK outperforms \texttt{fmincon} in terms of the number of problems successfully solved.
We expect this difference to be even more pronounced when further kinematic constraints are introduced, as the IK problem will admit more local minima.
%
	%
%In every observed case this difference was several orders of magnitude, serving as a certificate that a feasible solution exists for every problem in the problem set.
%%
%All the problems solved by our method for each manipulator tested resulted in duality gaps lower than 1, whereas gaps in infeasible problems were typically higher than $10^6$.

\section{Conclusion and Future Work}
\label{sec:conclusion} 

In this paper, we developed a novel and elegant formulation and solution of the IK problem for redundant manipulators. Our formulation of IK as a nearest point problem to a quadratic variety enabled us to prove the existence of problem instances admitting tight convex relaxations. Our use of convex relaxations provides certificates of global optimality alongside solutions. Our experiments demonstrated that the tight cases predicted by Theorem \ref{thm:stability} encompass many practical situations which can be efficiently solved via standard interior point methods. Furthermore, we empirically demonstrated that convex relaxations are efficient tools for reliably determining the (in)feasibility of IK problems; this is in stark contrast to local solver-based methods that need frequent re-starts and sampling-based methods that scale inefficiently with the number of joints. 

The tools presented in this paper hold promise for a variety of robotic manipulation and planning tasks. Careful selection of the nearest-point could be incorporated with task-specific goals like obstacle avoidance or low-energy motion planning. Methods for potentially extracting solutions from Sparse-BSOS solutions with rank greater than 1 warrant investigation. Additionally, theoretical tools developed in \cite{cifuentes2017local} and \cite{cifuentes2018geometry} could be used to precisely quantify the values of $\pmb{\xi}$ and goal poses for which SDP relaxations of our problem are tight. Using our method as a sub-solver in a branch-and-bound or mixed-integer nonlinear programming approach to inverse kinematics similar to \cite{dai2017global} also deserves attention as a means of incorporating complex obstacle avoidance constraints into a fast and efficient IK solver with performance guarantees. 

Finally, the Sparse-BSOS solver used here is a generic MATLAB library that does not exploit structure specific to our problem or use performance optimizations available in lower-level languages. Faster performance could be achieved by considering a custom sparse SDP relaxation (e.g., similar to the one in \cite{nie2009sum}) for sparsity patterns specific to IK. For kinematic chains with tens or hundreds of degrees of freedom, it may also be fruitful to investigate the use of Burer-Monteiro methods \cite{boumal2018deterministic}, which can require less time and memory, instead of standard interior point solvers. 
\vspace*{-1mm}
\section*{Acknowledgements}
This research was supported in part by a Dean's Catalyst Professorship from the University of Toronto.
We also gratefully acknowledge the support of the Natural Sciences and Engineering Research Council of Canada (NSERC) and the Ministry of Science and Education of the Republic of Croatia under the FLAG-ERA JTC 2016 project "RoboCom++".

\bibliographystyle{IEEEtran} % use IEEEtran.bst style
\bibliography{feasibility} % our bib file plus IEEE abreviation strings

\end{document}

% Define the basic page style
\fancypagestyle{plain}{%
    \fancyhf{}%
    \fancyfoot[C]{}%
    \fancyhead[R]{\begin{tabular}[b]{r}\small\sf \UTIASdocument\\
        \small\sf\UTIASrevision\\
        \small\sf\today \end{tabular}}%
    \fancyhead[L]{\includegraphics[height=0.6in]%
        {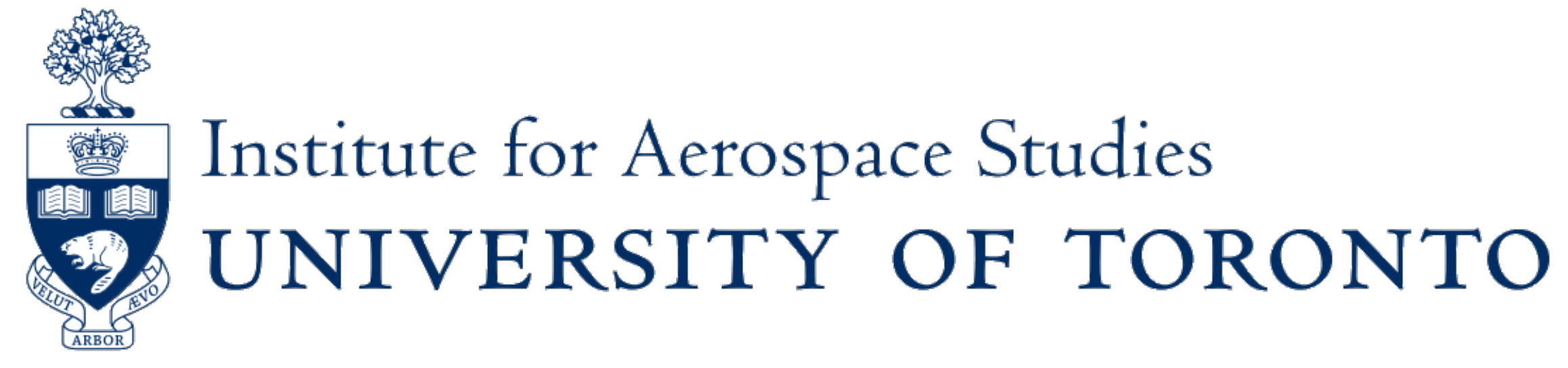} }%
    \renewcommand{\headrulewidth}{0pt}
    \renewcommand{\footrulewidth}{0pt}}

% Set the page style for the document
\pagestyle{fancy}

% Set the section labeling font
\allsectionsfont{\sf\bfseries}

% Set the caption labeling font
\renewcommand{\captionlabelfont}{\sf\bfseries}

%%%%%%%%%%%%%%%%%%%%%%%%%%%%%%%%%%%%%%%%%%%%%%%%%%%%%%%%%%%%%%%%%%%%%%%%%%%%%%%%
%% Headers and footers
%%%%%%%%%%%%%%%%%%%%%%%%%%%%%%%%%%%%%%%%%%%%%%%%%%%%%%%%%%%%%%%%%%%%%%%%%%%%%%%%
\lhead{ \includegraphics[height=0.6in]%
        {fig/utias_blue_on_clear.pdf} }
\rhead{ \begin{tabular}[b]{r}\small\sf \UTIASdocument\\
        \small\sf\UTIASrevision\\
        \small\sf\today \end{tabular}}
\chead{}
\lfoot{}
\cfoot{\thepage}
\rfoot{}

\renewcommand{\headrulewidth}{0pt}
\renewcommand{\footrulewidth}{0pt}

\newcommand{\kinematicsY}{Y_{\text{IK}}}
\newcommand{\fJacobian}{\nabla f(\bar{\pmb{\xi}}))}
\newcommand{\acq}{\mathrm{rank}(\fJacobian) = n - \mathrm{dim}_{\bar{\pmb{\xi}}} Y} 
\newcommand{\acqIK}{\mathrm{rank}(\fJacobian) = n - \mathrm{dim}_{\bar{\pmb{\xi}}} \kinematicsY} 
\newcommand{\acqIKPose}{\mathrm{rank}(\fJacobian) = n - \mathrm{dim}_{\bar{\pmb{\xi}}} \kinematicsY'}

%%%%%%%%%%%%%%%%%%%%%%%%%%%%%%%%%%%%%%%%%%%%%%%%%%%%%%%%%%%%%%%%%%%%%%%%%%%%%%%%
%% USER INPUT: Abstract
%%%%%%%%%%%%%%%%%%%%%%%%%%%%%%%%%%%%%%%%%%%%%%%%%%%%%%%%%%%%%%%%%%%%%%%%%%%%%%%%
\maketitle%

\begin{abstract}
  This document contains supplementary material for the paper titled ``Inverse Kinematics for Serial Kinematic Chains via Sum of Squares Optimization". A proof of Theorem 2 from the main paper is presented. It also contains a table with kinematic chain parameters used for simulation experiments reported on in the main paper.
\end{abstract}

%%%%%%%%%%%%%%%%%%%%%%%%%%%%%%%%%%%%%%%%%%%%%%%%%%%%%%%%%%%%%%%%%%%%%%%%%%%%%%%%
%% USER INPUT: Document Text
%%%%%%%%%%%%%%%%%%%%%%%%%%%%%%%%%%%%%%%%%%%%%%%%%%%%%%%%%%%%%%%%%%%%%%%%%%%%%%%%
\section{Proof of Theorem 2}
We begin by restating Theorems 1 \cite{cifuentes2017local} and 2 from the main paper.

\begin{theorem}[Nearest Point to a Quadratic Variety \cite{cifuentes2017local}] \label{thm:nearestPoint}
	Consider the problem 
	\begin{equation}
		\underset{\mbf{y} \in Y}{\min} \left\lVert\mbf{y} - \pmb{\xi}\right\rVert^2,
	\end{equation}
	where $Y \coloneqq \{\mbf{y} \in \mathbb{R}^n : f_1(\mbf{y}) = \cdots = f_m(\mbf{y}) = 0\}$,  $f_i$ quadratic.
	Let $\bar{\pmb{\xi}} \in Y$ be such that 
	\begin{equation}
		\acq.
	\end{equation}
	Then there is zero-duality-gap for any $\pmb{\xi} \in \mathbb{R}^n$ that is sufficiently close to $\bar{\pmb{\xi}}$.  
\end{theorem}

\begin{theorem}[Strong Duality] \label{thm:stability}
	If $\bar{\pmb{\xi}} \in \kinematicsY$ does not represent a fully extended configuration (i.e., the joint positions are not all collinear), and does not have any joints at their angular limits for specified base and goal positions, then the nearest point inverse kinematics QCQP in the main paper exhibits strong duality for all $\pmb{\xi}$ sufficiently close to $\bar{\pmb{\xi}}$.   
\end{theorem}

\begin{spacing}{1.5}
\begin{proof}
According to Theorem 1 in the main paper, it is sufficient to show that 	$\acqIK$ holds. The number of variables $n = d(N-1) + N$ scales with dimension $d \in \{2, 3\}$ and the number of links $N$. Theorem 1.7 in \cite{milgram2004geometry} gives us $\mathrm{dim}_{\bar{\pmb{\xi}}} \kinematicsY = (d-1)(N-1) - 1$. Therefore, we need to show that 
\begin{align*}
\text{rank}(\fJacobian) &= d(N-1) + N - (d-1)(N-1) + 1\\
	 &= 2N.	
\end{align*}
The structure of the Jacobian matrix $\fJacobian \in \mathbb{R}^{2N \times (d(N-1)+N)}$ can be understood in terms of $N$ link length constraints representing the first $N$ rows, and $N$ joint limit constraints representing the final $N$ rows:

\begin{equation}
	\fJacobian = 
	\begin{bmatrix}
		\mbf{J}_{1,1} & \mbf{0}_{N\times N}\\
		\mbf{J}_{2,1} & \mbf{J}_{2,2}
	\end{bmatrix}.
\end{equation}
The block lower triangular structure is due to the independence of link length constraints on $\mbf{s}$. Since $\rank(\fJacobian) \geq \rank(\mbf{J}_{1,1}) + \rank(\mbf{J}_{2,2})$ for block lower triangular matrices, it is sufficient to demonstrate that  $\rank(\mbf{J}_{1,1}) = \rank(\mbf{J}_{2,2}) = N$. Since $\mbf{J}_{2,2} = \diag(2\mbf{s})$, and $s_i > 0 \ \forall \ i \in \{1, \ldots, N\}$ by assumption, $\rank(\mbf{J}_{2,2}) = N$. It remains to demonstrate that 

\begin{equation}
\mbf{J}_{1,1}^T = 2
\begin{bmatrix}
	\mbf{x}_1 {-} \mbf{x}_0 & \mbf{x}_1 {-} \mbf{x}_2 & \mbf{0} & \cdots & \mbf{0} \\
	 \mbf{0} & \mbf{x}_2 {-} \mbf{x}_1 & \mbf{x}_2 {-} \mbf{x}_3 & \mbf{0} & \vdots \\
	 \vdots & \mbf{0} & \ddots & \ddots & \mbf{0} \\
	 \mbf{0} & \cdots & \mbf{0} & \mbf{x}_{N-1} {-} \mbf{x}_{N-2} & \mbf{x}_{N-1} {-} \mbf{x}_N
\end{bmatrix}
\end{equation}
has (full) rank $N$. Suppose $\rank(\mbf{J}_{1,1}) \neq N$. Then there exists $\mbf{v} \in \mathbb{R}^N$ such that $\mbf{v} \neq \mbf{0}$ and $\mbf{J}_{1,1}^T \mbf{v} = \mbf{0}$. This implies that 
\begin{equation} \label{eq:rankProof}
	v_i(\mbf{x}_i - \mbf{x}_{i-1}) = v_{i+1}(\mbf{x}_{i+1} - \mbf{x}_i), \ \ \forall i = 1, \ldots, N-1.
\end{equation}
Since there exists some $i$ such that $v_i \neq 0$, and the link lengths $l_i = \left\lVert \mbf{x}_i - \mbf{x}_{i-1}\right\rVert$ are all greater than zero, Equation \ref{eq:rankProof} tells us that $v_i \neq 0$ for all $i$. Therefore, $\mbf{x}_i - \mbf{x}_{i-1} = c_{ij}(\mbf{x}_j - \mbf{x}_{j-1})$ for all valid pairs of $i, j$, where  
	\begin{equation}
		c_{ij} = \frac{v_j}{v_i} \neq 0.
	\end{equation}
In other words, the link orientations are all collinear, which contradicts the assumption that the arm is not fully extended. 	Therefore, $\rank(\mbf{J}_{1,1}) = N$ and $\rank(\fJacobian) = 2N$, completing the proof.

\end{proof}

\begin{corollary}
	Theorem \ref{thm:stability} holds when the orientation of the end effector is also constrained. 
\end{corollary}
\begin{proof}
	Let $\kinematicsY'$ be the variety $\kinematicsY$ with an additional variable $s'$ and an additional constraint on the angle between $\hat{\mbf{z}}_N$ and $\hat{\mbf{z}}_{N+1} = \frac{1}{l_{N+1}}(\mbf{x}_{N+1} - \mbf{x}_N)$. That is, we additionally add the constraint 
	\begin{equation}
		\left\lVert \hat{\mbf{z}}_{N+1} - \hat{\mbf{z}}_{N}\right\rVert^2 + s'^2 -
  2\left(1 - \cos\alpha_{N+1}\right) = 0
	\end{equation}
	to our variety $Y_{IK}'$, where $\mbf{x}_N$ and $\mbf{x}_{N+1}$ are the fixed specified locations of the final joint and the end effector respectively, and $\alpha_{N+1}$ is the angular limit for the final joint. This variety $\kinematicsY'$ describing the feasible configurations for an $(N+1)$-DoF kinematic chain with its end effector and orientation fixed is therefore equivalent to $\kinematicsY$ for an $N$-DoF kinematic chain, with one variable and one constraint added. Recall that we need to prove that $\acqIKPose$. The constraint cannot affect $\mathrm{dim}_{\bar{\pmb{\xi}}} \kinematicsY'$ because we assume in our premise that none of the joint angles are activating their constraints, including the angular constraint corresponding to $s'$. The number of variables $n$ is simply increased by 1. Therefore, we must simply show that $\text{rank}(\fJacobian)$ increases by 1 with the additional variable and constraint added. This follows immediately from the fact that the additional constraint increases the rank of $\rank(\mbf{J}_{2,2})$ by 1 and does not affect $\rank(\mbf{J}_{1,1})$. 
\end{proof}

\section{Parameter Table}
\begin{table}[h!]
%  \begin{center}
	\centering
        \caption{Parameters used for the experiments in the main paper.}
        \vspace{-2mm}
    \label{tab:rankOneParams}
    \begin{tabular}{ccccccccccc} % 
    	joint & 1 & 2 & 3 & 4 & 5 & 6 & 7 & 8 & 9 & 10 \\
    	\midrule
        $|\theta_{\text{i}}|_{\tiny \text{max}}$& $\frac{\pi}{4}$ & $\frac{\pi}{4}$ & $\frac{\pi}{8}$ & $\frac{\pi}{4}$ & $\frac{\pi}{4}$ & $\frac{\pi}{2}$ & $\frac{\pi}{4}$ & $\frac{\pi}{4}$ & $\frac{\pi}{2}$ & $\frac{\pi}{8}$\\
        $l_i$ & 2 & 2 & 1 & 2 & 3 & 2 & 4 & 4 & 1 & 2\\
    \end{tabular}
%  \end{center}
   
\end{table}

\end{spacing}

%%%%%%%%%%%%%%%%%%%%%%%%%%%%%%%%%%%%%%%%%%%%%%%%%%%%%%%%%%%%%%%%%%%%%%%%%%%%%%%%
%% USER INPUT: Bibliography
%%%%%%%%%%%%%%%%%%%%%%%%%%%%%%%%%%%%%%%%%%%%%%%%%%%%%%%%%%%%%%%%%%%%%%%%%%%%%%%%
%\newpage
\bibliographystyle{plain}
\bibliography{../feasibility}